\documentclass[english,a4paper,12pt]{article}

\usepackage{amsxtra, amsfonts, amssymb, amstext, amsmath, mathtools}
\usepackage{amsthm}
\usepackage{booktabs}
\usepackage{nicefrac}
\usepackage{xspace}
\usepackage{url}
\usepackage{graphics,color}
\usepackage[algo2e,ruled,vlined,linesnumbered]{algorithm2e}
\usepackage{wrapfig}

\newtheorem{theorem}{Theorem}
\newtheorem{lemma}[theorem]{Lemma}

\newcommand{\oea}{\mbox{$(1 + 1)$~EA}\xspace}

\newcommand{\OM}{\textsc{OneMax}\xspace}
\newcommand{\onemax}{\OM}
\newcommand{\LO}{\textsc{Leading\-Ones}\xspace}
\newcommand{\leadingones}{\LO}

\newcommand{\binval}{\textsc{BinVal}\xspace}
\DeclareMathOperator{\jump}{\textsc{Jump}}

\DeclareMathOperator{\Sample}{Sample}
\DeclareMathOperator{\minmax}{minmax}
\DeclareMathOperator{\paral}{par}
\DeclareMathOperator{\poly}{poly}

\newcommand{\R}{\ensuremath{\mathbb{R}}}

\newcommand{\N}{\ensuremath{\mathbb{N}}} 

\newcommand{\calA}{\ensuremath{\mathcal{A}}} 
\newcommand{\calF}{\ensuremath{\mathcal{F}}} 
\newcommand{\calP}{\ensuremath{\mathcal{P}}}

\DeclareMathOperator{\Bin}{Bin}

\newcommand{\eps}{\varepsilon}

\newcommand{\assign}{\leftarrow}

\begin{document}

\title{A Tight Runtime Analysis for the cGA on Jump Functions---EDAs Can Cross Fitness Valleys at No Extra Cost\thanks{Full version of a paper~\cite{Doerr19} appearing at GECCO 2019.}}

\author{Benjamin Doerr\\ \'Ecole Polytechnique\\ CNRS\\ Laboratoire d'Informatique (LIX)\\ Palaiseau\\ France}

\maketitle

\sloppy{
\begin{abstract}
  We prove that the compact genetic algorithm (cGA) with hypothetical population size $\mu = \Omega(\sqrt n \log n) \cap \poly(n)$ with high probability finds the optimum of any $n$-dimensional jump function with jump size $k < \frac 1 {20} \ln n$ in $O(\mu \sqrt n)$ iterations. Since it is known that the cGA with high probability needs at least $\Omega(\mu \sqrt n + n \log n)$ iterations to optimize the unimodal $\onemax$ function, our result shows that the cGA in contrast to most classic evolutionary algorithms here is able to cross moderate-sized valleys of low fitness at no extra cost. 
  
  Our runtime guarantee improves over the recent upper bound $O(\mu n^{1.5} \log n)$ valid for $\mu = \Omega(n^{3.5+\eps})$ of Hasen\"ohrl and Sutton (GECCO 2018). For the best choice of the hypothetical population size, this result gives a runtime guarantee of $O(n^{5+\eps})$, whereas ours gives $O(n \log n)$.  
  
  We also provide a simple general method based on parallel runs that, under mild conditions, (i)~overcomes the need to specify a suitable population size, but gives a performance close to the one stemming from the best-possible population size, and (ii)~transforms EDAs with high-probability performance guarantees into EDAs with similar bounds on the expected runtime.
  
\end{abstract}

\section{Introduction}

While the mathematical analysis of evolutionary algorithms (EAs) has produced a plethora of insightful results in the last over 20 years, the rigorous understanding of estimation-of-distribution algorithms (EDAs) is much less developed~\cite{KrejcaW18}. Obviously, this is due to the highly complex stochastic processes that describe the runs of such algorithms. In consequence, despite significant efforts and deep results~\cite{Droste06,SudholtW16,LenglerSW18}, not even the runtime of the compact genetic algorithm (cGA) on the \onemax benchmark function is fully understood (here we would argue that the cGA is the most simple EDA and that the unimodal \onemax function, counting the number of ones in a bit string, is the most simple optimization problem with unique global optimum). It is therefore not surprising that many questions which are well-understood for EAs are only started being understood for EDAs. One such question is how EDAs optimize objective functions that are not unimodal.

In the first and so far only runtime analysis of an EDA on a non-unimodal objective function, Hasen\"ohrl and Sutton~\cite{HasenohrlS18} regard the optimization time of the cGA on the jump function class, which are unimodal apart from having a valley of low fitness of scalable size $k$ around the global optimum. They show~\cite[Theorem~3.3]{HasenohrlS18} that, for a sufficiently large constant $C$ and any constant $\eps > 0$, the cGA with hypothetical population size at least $\mu \ge \max\{C n e^{4k}, n^{3.5+\eps}\}$\footnote{In the paper, this is stated as minimum of the two terms, but from the proofs it is clear that it should be the maximum.} with probability $1 - o(1)$ finds the optimum of any jump function with jump size at most $k = o(n)$ in $O(\mu n^{1.5} \log n + e^{4k})$ generations (which is also the number of fitness evaluations, since the cGA evaluates only two search points in each iteration).

This result is remarkable in that it shows that the cGA with the right choice of $\mu$ and for $k \ge 6$ is more efficient on jump functions than most evolutionary algorithms, who have a runtime of at least $\Omega(n^k)$, see Section~\ref{sec:jump}. 

There is one aspect in which the result of Hasen\"ohrl and Sutton is not yet perfect (and this is the motivation of this work). We note that even when choosing the smallest possible population size $\mu = n^{3.5+\eps}$, the runtime guarantee becomes at least $\Omega(n^{5+\eps})$. While clearly a polynomial runtime, and thus \emph{efficient} in the classic complexity theory view, this is a runtime that is not practical in many applications (and we recall here that the target of the mathematical analysis of evolutionary algorithms is not to understand jump functions, but to derive from the analysis on simple test problems insight that extend to practically relevant problems). Also, this runtime guarantee is weaker than the $O(n^k)$ bound for simple mutation-based EAs such as the \oea when $k \le 5$. Hence one could feel that the result of Hasen\"ohrl and Sutton shows the superiority of EDAs rather for problem instances for which both the runtime of typical EAs and the performance guarantee for the cGA are prohibitively large. In a similar vein, one has to question if a practitioner would run the cGA with a hypothetical population size of more than $n^{3.5}$ when solving a problem defined over bit strings of length $n$.

\textbf{Our main result} is that these potential weaknesses of the cGA are not real and that the cGA performs in fact much better than what the previous work shows. We prove rigorously that the cGA with hypothetical population size $\mu \ge K \sqrt n \log n$, $K$ a sufficiently large constant, and $\mu$ polynomially bounded in $n$ with high probability optimizes any $n$-dimensional jump function with jump size $k < \frac 1 {20} \ln n$ in only $O(\mu \sqrt n)$ iterations. For the smallest admissible populations size $\mu = \Theta(\sqrt n \log n)$, this gives a runtime guarantee of $O(n \log n)$, a result that both overcomes the large runtime and the large required hypothetical population size of the previous result.

From a broader perspective our result shows that the cGA (and we expect similar result to hold for other EDAs) does not suffer from moderate-size valleys of low fitness. We recall that Sudholt and Witt~\cite{SudholtW16} have shown that the cGA with any hypothetical population size (polynomial in $n$) with high probability needs $\Omega(\mu\sqrt n + n \log n)$ iterations to optimize the \onemax function. Hence our result shows that adding a valley of low fitness to the \onemax function does not worsen the asymptotic performance of the cGA as long as the fitness valley is smaller than $\frac 1 {20} \ln n$.

We may add that our work also makes some arguments of~\cite{HasenohrlS18} more rigorous. In particular, we observe that the progress of the cGA cannot be estimated by taking the progress one would have when no fitness valley was present and correcting this estimate by inverting the progress with the probability that a search point is sampled in the fitness valley. This argument ignores the stochastic dependencies between the absolute value of the progress and the event that a solution in the gap is sampled. These dependencies are real and have a negative impact as discussed in more detail before Lemma~\ref{ldrift}.

We note that the approach of intentionally ignoring some dependencies to make a mathematical analysis tractable, often called mean-field analysis, is common in some scientific areas, most notably statistical physics, and has also been used in evolutionary computation, e.g.,~\cite{ZhengYD18}. This approach, however, needs an additional justification, e.g., via specific experiments, why the omission of the dependencies should not change the matter substantially. In any case, such mean-field approaches do not lead to results fully proven in the mathematical sense. In this sense, we hope that our work also provides methods that help in future analyses of EDAs on non-unimodal optimization problems.

As a \textbf{side result}, triggered by the fact that we ``only'' show a bound that holds with high probability, but not a bound on the expected runtime, we provide a general approach to transform an EDA using a population size parameter $\mu$ into an algorithm that does not require the specification of such a parameter, but has a performance similar to the one of the EDA with optimally chosen parameter. This performance guarantee also holds for the expected runtime, even if for the EDA only a with-high-probability runtime guarantee is known.

\section{Preliminaries}

\subsection{The Compact Genetic Algorithm}

The \emph{compact genetic algorithm} (cGA) is an estimation-of-distribution algorithm (EDA) proposed by Harik, Lobo, and Goldberg~\cite{HarikLG99} for the maximization of pseudo-Boolean functions $\calF : \{0,1\}^n \to \R$. Being a univariate EDA, it develops a probabilistic model described by a frequency vector $f \in [0,1]^n$. This frequency vector describes a probability distribution on the search space $\{0,1\}^n$. If $X = (X_1, \dots, X_n) \in \{0,1\}^n$ is a search point sampled according to this distribution---we write \[X \sim \Sample(f)\] to indicate this---then we have $\Pr[X_i = 1] = f_i$ independently for all $i \in [1..n] \coloneqq \{1, \dots, n\}$. In other words, the probability that $X$ equals some fixed search point $y$ is 
\[\Pr[X = y] = \prod_{i : y_i = 1} f_i \prod_{i : y_i = 0} (1 - f_i).\]

In each iteration, the cGA updates this probabilistic model as follows. It samples two search points $x^1, x^2 \sim \Sample(f)$, computes the fitness of both, and defines $(y^1,y^2) = (x^1,x^2)$ when $x^1$ is at least as fit as $x^2$ and $(y^1,y^2) = (x^2,x^1)$ otherwise. Consequently, $y^1$ is the rather better search point of the two. We then define a preliminary model by $f' \coloneqq f + \frac 1 \mu (y^1 - y^2)$. This definition ensures that, when $y^1$ and $y^2$ differ in some bit position $i$, the $i$-th preliminary frequency moves by a step of $\frac 1 \mu$ into the direction of $y^1_i$, which we hope to be the right direction since $y^1$ is the better of the two search points. The \emph{hypothetical populations size}~$\mu$ is used to control how strong this update is. 

To avoid a premature convergence, we ensure that the new frequency vector is in $[\frac 1n, 1 - \frac 1n]^n$ by capping too small or too large values at the corresponding boundaries. More precisely, for all $\ell \le u$ and all $r \in \R$ we define 
\[
\minmax(\ell,r,u) \coloneqq \max\{\ell,\min\{r,u\}\} = \begin{cases} 
\ell & \mbox{if $r < \ell$}\\
r & \mbox{if $r \in [\ell,u]$}\\
u & \mbox{if $r > u$}
\end{cases}
\] 
and we lift this notation to vectors by reading it component-wise. Now the new frequency vector is $\minmax(\frac 1n \mathbf{1}_n, f', (1 - \frac 1n) \mathbf{1}_n)$.

This iterative frequency development is pursued until some termination criterion is met. Since we aim at analyzing the time (number of iterations) it takes to sample the optimal solution (this is what we call the \emph{runtime} of the cGA), we do not specify a termination criterion and pretend that the algorithm runs forever.

The pseudo-code for the cGA is given in Algorithm~\ref{alg:cga}. We shall use the notation given there frequently in our proofs. For the frequency vector $f_t$ obtained at the end of iteration $t$, we denote its $i$-th component by $f_{i,t}$ or, when there is no risk of ambiguity, by $f_{it}$.
	
\begin{algorithm2e}%
	$t \assign 0$\;
	$f_t = (\frac 12, \dots, \frac 12) \in [0,1]^n$\;
	\Repeat{forever}{
    $x^1 \assign \Sample(f_t)$\;
    $x^2 \assign \Sample(f_t)$\;
    \leIf{$\calF(x^1) \ge \calF(x^2)$}{$(y^1,y^2) \assign (x^1,x^2)$}{$(y^1,y^2) \assign (x^2,x^1)$}
    $f'_{t+1} \assign f_t + \frac 1 \mu (y^1-y^2)$\;
    $f_{t+1} \assign \minmax(\frac 1n \mathbf{1}_n, f'_{t+1}, (1 - \frac 1n) \mathbf{1}_n)$\;
    $t \assign t+1$\; 
  }
\caption{The compact genetic algorithm (cGA) to maximize a function $\calF : \{0,1\}^n \to \R$.}
\label{alg:cga}
\end{algorithm2e}

\textbf{Well-behaved frequency assumption:} 
For the hypothetical population size $\mu$, we take the common assumption that any two frequencies that can occur in a run of the cGA differ by a multiple of $\frac 1 \mu$. We call this the \emph{well-behaved frequency assumption}. This assumption was implicitly already made in~\cite{HarikLG99} by using even $\mu$ in all experiments (note that the hypothetical population size is denoted by $n$ in~\cite{HarikLG99}). This assumption was made explicit in~\cite{Droste06} by requiring $\mu$ to be even. Both works do not use the frequencies boundaries $\frac 1n$ and $1 - \frac 1n$, so an even value for $\mu$ ensures well-behaved frequencies. 

For the case with frequency boundaries, the well-behaved frequency assumption is equivalent to $(1-\frac 2n)$ being an even multiple of the update step size $\frac 1 \mu$. In this case, $n_\mu = (1 - \frac 2n) \mu \in 2 \N$ and the set of frequencies that can occur is \[F \coloneqq F_\mu \coloneqq \{\tfrac 1n + \tfrac i \mu \mid i \in [0..n_\mu]\}.\] 
This assumption was made, e.g., in the proof of Theorem~2 in~\cite{SudholtW16} and in the paper~\cite{LenglerSW18} (see the paragraph following Lemma~2.1).

\subsection{Related Work}\label{sec:jump}

In all results described in this section, we shall assume that the hypothetical population size is at most polynomial in the problem size $n$, that is, that there is a constant $c$ such that $\mu \le n^c$. 

The first to conduct a rigorous runtime analysis for the cGA was Droste in his seminal work~\cite{Droste06}. He regarded the cGA without frequency boundaries, that is, he just took $f_{t+1} \coloneqq f'_{t+1}$ in our notation. He showed that this algorithm with $\mu \ge n^{1/2 + \eps}$, $\eps > 0$ any positive constant, finds the optimum of the $\onemax$ function defined by 
\[\onemax(x) = \|x\|_1 = \sum_{i=1}^n x_i\] 
for all $x \in \{0,1\}^n$ with probability at least $1/2$ in $O(\mu \sqrt n)$ iterations~\cite[Theorem~8]{Droste06}. 

Droste also showed that this cGA for any objective function $\calF$ with unique optimum has an expected runtime of $\Omega(\mu \sqrt n)$ when conditioning on no premature convergence~\cite[Theorem~6]{Droste06}. It is easy to see that his proof of the lower bound can be extended to the cGA with frequency boundaries, that is, to Algorithm~\ref{alg:cga}. For this, it suffices to deduce from his drift argument the result that the first time $T_{n/4}$ that the frequency distance $D = \sum_{i=1}^n (1 - f_{it})$ is less than $n/4$ satisfies $E[T_{n/4}] \ge \mu \sqrt n \frac{\sqrt 2}{4}$. Since the probability to sample the optimum from a frequency distance of at least $n/4$ is at most 
\begin{align*}
\prod_{i=1}^n f_{it} &= \prod_{i=1}^n (1 - (1 - f_{it})) \le \prod_{i=1}^n \exp(-(1 - f_{it})) \\
&= \exp\left(-\sum_{i=1}^n (1-f_{it})\right) \le \exp(-n/4),
\end{align*} 
the algorithm with high probability does not find the optimum before time~$T_{n/4}$.

Ten years after Droste's work, Sudholt and Witt~\cite{SudholtW16} showed that the $O(\mu \sqrt n)$ upper bound also holds for the cGA with frequency boundaries. There (but the same should be true for the cGA without boundaries) a hypothetical population size of $\mu =\Omega(\sqrt n \log n)$ suffices (recall that Droste required $\mu = \Omega(n^{1/2+\eps})$). The technically biggest progress with respect to upper bounds most likely lies in the fact that the analysis in~\cite{SudholtW16} also holds for the expected optimization time, which means that it also includes the rare case that frequencies reach the lower boundary (see our discussion of the relation of expectations and tail bounds for runtimes of EDAs in Section~\ref{sec:whp}). Sudholt and Witt also show that the cGA with frequency boundaries with high probability (and thus also in expectation) needs at least $\Omega(\mu\sqrt n + n \log n)$ iterations to optimize $\onemax$. While the $\mu\sqrt n$ lower bound could have been also obtained with methods similar to Droste's (in Lemma~\ref{lonemax} we do something very similar), the innocent-looking $\Omega(n \log n)$ bound is surprisingly difficult to prove.

Not much is known for hypothetical population sizes below the order of $\sqrt n$. It is clear that then the frequencies will reach the lower boundary of the frequency range, so working with a non-trivial lower boundary like~$\frac 1n$ is necessary to prevent premature convergence. The recent lower bound $\Omega(\mu^{1/3} n)$ valid for $\mu = O(\frac{\sqrt n}{\log n \log\log n})$ of~\cite{LenglerSW18} indicates that already a little below the $\sqrt n$ regime significantly larger runtimes occur, but with no upper bounds this regime remains largely not understood.

We refer the reader to the recent survey~\cite{KrejcaW18} for more results on the runtime of the cGA on classic unimodal test functions like \leadingones and \binval. Interestingly, nothing was known for non-unimodal functions before the recent work of Hasen\"ohrl and Sutton~\cite{HasenohrlS18} on jump functions, which we discussed already in the introduction. 

To round off the picture, we briefly describe some typical runtimes of evolutionary algorithms on jump functions. We recall that the $n$-dimensional jump function with jump size $k \ge 1$ is defined by
\[
\jump_{nk}(x) = 
\begin{cases}
\|x\|_1+k & \mbox{if $\|x\|_1 \in [0..n-k] \cup \{n\}$,}\\
n - \|x\|_1 & \mbox{if $\|x\|_1 \in [n-k+1\, ..\, n-1]$}.
\end{cases}
\]
Hence for $k = 1$, we have a fitness landscape isomorphic to the one of $\onemax$, but for larger values of $k$ there is a fitness valley consisting of the $k-1$ highest sub-optimal fitness levels of the \onemax function. This valley is hard to cross for evolutionary algorithms using standard-bit mutation with mutation rate $\frac 1n$ since with very high probability they need to generate the optimum from one of the local optima, which in a single application of the mutation operator happens only with probability less than $n^{-k}$. For this reason, e.g., the classic $(\mu+\lambda)$ and $(\mu,\lambda)$ EAs all have a runtime of at least $n^k$. This was proven formally for the \oea in the classic paper~\cite{DrosteJW02}, but the argument just given proves the $n^k$ lower bound equally well for all $(\mu+\lambda)$ and $(\mu,\lambda)$ EAs. By using larger mutation rates or a heavy-tailed mutation operator, a $k^{\Theta(k)}$ runtime improvement can be obtained~\cite{DoerrLMN17}, but the runtime remains $\Omega(n^k)$ for $k$ constant. 

Asymptotically better runtimes can be achieved when using crossover, though this is harder than expected. The first work in this direction~\cite{JansenW02}, among other results, could show that a simple $(\mu+1)$ genetic algorithm using uniform crossover with rate $p_c = O(1 / kn)$ obtains an $O(\mu n^2 k^3 + 2^{2k} p_c^{-1})$ runtime when the population size is at least $\mu = \Omega(k \log n)$. A shortcoming of this result, already noted by the authors, is that it only applies to uncommonly small crossover rates. Using a different algorithm that first always applies crossover and then mutation, a runtime of $O(n^{k-1} \log n)$ was achieved by Dang et al.~\cite[Theorem~2]{DangFKKLOSS18}. For $k \ge 3$, the logarithmic factor in the runtime can be removed by using a higher mutation rate. With additional diversity mechanisms, the runtime can be further reduced up to $O(n \log n + 4^k)$, see~\cite{DangFKKLOSS16}. In the light of this last result, the insight stemming from the previous work~\cite{HasenohrlS18} and ours is that the cGA apparently without further modifications supplies the necessary diversity to obtain a runtime of $O(n \log n + 2^{O(k)})$.

Finally, we note that runtimes of $O(n \binom{n}{k})$ and $O(k \log(n) \binom{n}{k})$ were shown for the $(1+1)$~IA$^{\mathrm hyp}$ and the $(1+1)$ Fast-IA artificial immune systems, respectively~\cite{CorusOY17,CorusOY18fast}.

\subsection{Expected Runtimes versus Guarantees with High Probability}\label{sec:whp}

We note that our main result as well as the previous one~\cite{HasenohrlS18} for this problem give runtime bounds that hold with high probability, that is, with probability $1 - o(1)$. However, we do not show a bound on the expected runtime. Let us quickly argue what the differences are, why we chose to prove a high-probability statement, and how to transform EDAs with high-probability guarantees into those with guarantees on the expected runtime. We note that Wegener~\cite[Section~3]{Wegener05} with different arguments also suggests to prefer high-probability guarantees over expected runtimes.

For most evolutionary algorithms a high-probability guarantee can easily be turned into a bound on the expected runtime. If we know that a certain algorithm from any initial state finds the optimum in time $T$ with at least constant probability, then by splitting time into consecutive segments of length $T$ we see that after time $\gamma T$ the probability that the algorithm has not succeeded is at most $\exp(-\Omega(\gamma))$. Consequently, the runtime is stochastically dominated by $T$ times a geometric random variable with constant success rate, and consequently, the expected runtime is $O(T)$. The same argument gives a scalable tail bound of type ``with probability at most $\exp(-\Omega(\gamma))$, the runtime is more than $\gamma T$.''

For EDAs, it is usually much harder to show a good performance for any initial situation since there are some states which are particularly unfavorable (usually when all frequencies are close to the wrong boundary value). This does not rule out that the expected runtime and the time that is obtained with high probability are of the same order, but proving the bound on the expected runtime needs stronger arguments. The analysis of the expected runtime of the cGA on \onemax in~\cite{SudholtW16} is an example for such a result. 

This additional proof complexity raises the question if this effort is justified if the hardest part is dealing with states of the algorithm that are rarely reached (in~\cite{SudholtW16} with probability $O(n^{-c})$ only, where $c$ can be any positive constant). While we think that is was very valuable that the work~\cite{SudholtW16} showed how to compute expected runtimes for EDAs, we feel that such results are not always needed, both because of the difficulty to obtain such results and because, in some sense, they are a mildly unnatural remedy to the deeper problem.

As said, the main reason why guarantees for the expected runtime of an EDA can be difficult to show is that the EDA with small probability can reach a state from which the optimum is hard to reach. When in such a state, however, instead of spending much time to leave the unfavorable state, it would be more efficient and more natural to simply restart the algorithm and have a new good chance for a fast optimization process. While we cannot expect the algorithm to detect that it is in an unfavorable state (except in the case of premature convergence when no frequency boundaries are used), the following simple parallel-run strategy under mild assumptions can do this automatically. More precisely, via suitable parallel runs we obtain an expected runtime that is only a logarithmic factor above the runtime the EDA would have with high probability when using the optimal population size. Hence this approach both obtains expected runtimes and optimizes the value of the parameter $\mu$. \\

\noindent\textbf{Parallel EDA runs with exponentially growing population size:} Let $\calA$ be an EDA with a parameter $\mu$ and let $\calP$ be a problem we want to solve. We assume that there are unknown values $\tilde\mu$ and $T$ such that $\calA$ with any parameter value $\mu \ge \tilde\mu$ solves $\calP$ in time $\mu T$ with probability at least $\frac 34$.

We propose the following strategy to solve $\calP$ via parallel runs of $\calA$ with different parameter values. We start with no process running. In round $i = 1, 2, \dots$ of our strategy, we let all running processes (which are process $1$ to $i-1$) use a computational budget of $2^{i-1}$; further, we start process $i$ with parameter $\mu = \mu_i \coloneqq 2^{i-1}$ and let it use a budget of $\sum_{j=0}^{i-1} 2^j$. We stop when any process has solved the problem.

We observe that at the end of round $i$, processes $1$ to $i$ are running and have each spent a budget of $\sum_{j=0}^{i-1} 2^j$. Consequently, the total budget spent in the first $i$ rounds is less than $i 2^i$. 

Note that after round $i_0 = 1 + \lceil \log_2 \tilde\mu \rceil + \lfloor \log_2 T \rfloor$, the process with parameter $\mu = 2^{\lceil \log_2 \tilde\mu \rceil} \ge \tilde\mu$ has started and has used a time budget of 
\[\sum_{j=0}^{i_0 - 1} 2^j \ge \sum_{j=\lceil \log_2 \tilde\mu \rceil}^{i_0 - 1} 2^j = \mu \sum_{j=0}^{\lfloor \log_2 T \rfloor} 2^j \ge \mu T.\] 
Consequently, with probability $3/4$ this process has found the optimum at that time. With the same type of computation, we see that in round $i_0 + j$, the process with parameter $2^j \mu$ is finished with probability $3/4$. Consequently, the round in which we found the solution is dominated by $i_0 - 1$ plus a geometric distribution with success rate $3/4$. The expected time taken by this strategy to solve the problem thus is at most 
\[\sum_{i = i_0}^\infty \left(\frac 14\right)^{i - i_0} \left(\frac 34\right) i \, 2^i = \frac 34 \, 2^{i_0} \sum_{j=0}^\infty 2^{-j} (j+i_0) = 3 \cdot 2^{i_0-1} (i_0+1)\]
using the well-known result $\sum_{j=0}^\infty j \, 2^{-j} = 2$. We further estimate the expected runtime of our parallel-run strategy by
\[
3 \cdot 2^{i_0-1} (i_0+1) \le 3 \mu T (\log_2 (\mu T) + 2) \le 6 \tilde\mu T (\log_2 (\tilde\mu T) + 3) \eqqcolon T_{\paral}.
\]
We note that, again, analogous arguments give the scalable tail bound that with probability at most $\exp(-\Omega(\gamma))$, the runtime exceeds $\gamma T_{\paral}$. We recall here that for EDAs such tail bounds are usually not shown, again due to the fact that the EDA may reach a situation from which is takes a long time to reach the optimum.

We note that if the values of $\tilde \mu$ and $T$ were known in advance, then restarting the EDA with $\mu = \tilde \mu$ and with a budget of $T$ until the problem is solved would immediately give an algorithm with expected runtime $T^* \le \frac 43 \tilde \mu T$. This is the best-possible expected runtime that can be deduced from our assumptions. Consequently, our parallel-run strategy with its $O(T^* \log T^*)$ expected runtime obtains the optimal expected runtime apart from a logarithmic factor. 

We remark that a logarithmic factor usually is not a lot compared to what can be lost by choosing a wrong algorithm parameter, in particular, when the parameter is hard to guess. We note here that the recent work~\cite{LenglerSW18} suggests that already for the simple \onemax function, the hypothetical population size has a non-obvious influence on the runtime: Sufficiently small values give an $O(n \log n)$ runtime, in a middle regime the runtime increases to $\tilde\Omega(n^{7/6})$ before dropping again to $O(n \log n)$ and then increasing linearly with $\mu$. In the light of such results, a logarithmic overhead for exploiting a near-optimal rate appears to be a good trade-off.

\section{Main Technical Analysis}

We now conduct our runtime analysis of the cGA on jump functions. We start by giving a rough overview of the proof, then provide the necessary ingredients of the main proof, and finally state and prove our main result.

\subsection{Proof Overview}

We now give a brief overview of our runtime analysis and show how the different partial results work together. We leave it to the reader to read this section now or after the presentation of the partial results (or twice).

In our analysis, we roughly distinguish three phases of the optimization process. The first phase lasts until for the first time the frequency distance $D_t \coloneqq n - \|f_t\|_1$ is $O(\log n)$ with a large implicit constant. During this phase, by Lemma~\ref{lsample} and a union bound, with high probability we will never sample a solution in the gap. Consequently, we can pretend that we are optimizing the \onemax function and use our analysis of Lemma~\ref{lonemax}, which reuses arguments of the classic result by Droste~\cite{Droste06} including Lemma~\ref{ldroste}. The second phase then lasts until we have a $D_t$ value of $O(k)$, again with large implicit constant. In this phase, we use the drift computed in Lemma~\ref{ldrift}. We profit from the fact that in this phase we only need to obtain a moderate decrease of $D_t$ and apply the additive drift theorem with the smallest drift that can occur in this phase, which is $\Omega(1/\mu)$. Since this phase is so short, a simple Markov bound suffices to show that the phase ends with high probability in due time. Once we reach a $D_t$ value of $O(k)$, we have a reasonable chance to sample the optimum by Lemma~\ref{lopt}. Since in this phase samples in the gap occur frequently, we have less control over $D_t$, in particular, we cannot exhibit an expected decrease of~$D_t$. We therefore pessimistically estimate $D_t$ as if $D_t$ would always increase, which gives (apart from the boundary effects described in Lemma~\ref{lboundary}) an increase of $| \|x^1\|_1 - \|x^2\|_1 |$. Since $D_t$ is small, these increases are small as well, as again ensured by Lemma~\ref{lsample}. With this observation, we can argue that we have a $D_t$ value of $O(k)$ for almost $\mu$ iterations, which suffices to sample the optimum with high probability.

All the arguments above need that the frequencies are bounded away from the lower boundary of $\frac 1n$, more precisely, that they are $\Omega(1)$ at all times. In the first two phases, we ensure this via Lemma~\ref{lconc}, our general result for random processes that are not Markov processes. To this aim, we estimate the probabilities of certain frequency changes by adjusting this data from the \onemax process (Lemma~\ref{lonemax2}, taken from  Sudholt and Witt~\cite{SudholtW16}) via a pessimistic estimate of the negative influence of search points sampled in the gap. For the third phase, the fact that this phase only last $o(\mu)$ iterations implies that frequencies change by at most $o(1)$, hence the $\Omega(1)$ lower bound remains intact.

\subsection{Technical Ingredients of the Main Proof}

In this section, we collect the central arguments needed in the proof of our main result. Since we hope that some arguments are helpful for other runtime analyses of EDAs, we fix no general notation apart from the one defined in Algorithm~\ref{alg:cga} (at the price of occasionally restating a notion).

We frequently use the following estimate, which states that the \onemax fitness of a search point sampled from $\Sample(f)$ is close to the expected \onemax fitness $\|f\|_1$. Since we mostly need such results for frequency vectors close to $(1, \ldots, 1)$, we formulate this result in terms of distances to the maximum values. 

\begin{lemma}\label{lsample}
  Let $f \in [0,1]^n$, $D \coloneqq n - \|f\|_1$, $D^- \le D \le D^+$, $x \sim \Sample(f)$, and $d(x) \coloneqq n - \|x\|_1$. Then for all $\delta \in [0,1]$, we have
  \begin{align*}
  \Pr[d(x) \ge (1+\delta) D^+] & \le \exp(-\tfrac 13 \delta^2 D^+),\\
  \Pr[d(x) \le (1-\delta) D^-] & \le \exp(-\tfrac 12 \delta^2 D^-).
  \end{align*}
\end{lemma}

\begin{proof}
  The random variable $n-\|x\|_1$ can be written as a sum $n-\|x\|_1 \eqqcolon Z = \sum_{i=1}^n Z_i$ of $n$ independent binary random variables $Z_1, \dots, Z_n$ such that $\Pr[Z_i = 1] = 1 - f_{i}$. By definition, $E[Z] = D$. The claims follow directly from the classic multiplicative Chernoff bounds (Theorem~1 in~\cite{Hoeffding63} or, e.g., Theorems 10.1 and 10.5 together with Section~10.1.8 in the survey~\cite{Doerr18bookchapter}).
\end{proof}

We need the lemma above in particular to argue that the probability to sample a search point in the gap region of the $\jump$ function is small. For the $\jump_{nk}$ function, we observe that when $D \coloneqq n - \|f\|_1$ is at least $2k$, then the probability that $x \sim \Sample(f)$ lies in the gap, that is, has $n-k < \|x\|_1 < n$, is $e^{-\Omega(k)}$. We can also get low constant probabilities for sampling in the gap when $D \ge k + \Omega(\sqrt k)$ with large implicit constant. In~\cite[Lemma~3.2]{HasenohrlS18}, a gap probability of at most $1 - 1/\sqrt 2 \le 0.293$ is shown when $D \ge k+c$ for $c$ a sufficiently large constant and $k = o(n)$, but we are skeptical that this is true. Note that when $f = \frac{n-k-c}{n} \mathbf{1}_n$, then $X = n - \|x\|_1$ with $x \sim \Sample(f)$ follows a binomial distribution with parameters $n$ and $\frac{k+c}{n}$. Hence if $k$ is large compared to $c$, then $\Pr[X < k] = \Pr[X < E[X] - c] \approx \frac 12$.

When, in the notation of Algorithm~\ref{alg:cga}, the current frequency vector $f_t$ is such that $f_{it} \in \{\frac 1n, 1 - \frac 1n\}$ for some $i \in [1..n]$, then it may happen that $f'_{t+1} \notin [\frac 1n, 1-\frac 1n]$ and consequently $f_{t+1}$ does not satisfy the nice relation $f_{t+1} = f_t + \frac 1 \mu (y^1 - y^2)$. The following lemma quantifies these discrepancies.
 
\begin{lemma}\label{lboundary}
  Let $P = 2 \frac 1n (1-\frac 1n)$. Let $t \ge 0$. Using the notation given in Algorithm~\ref{alg:cga}, consider iteration $t+1$ of a run of the cGA started with a fixed frequency vector $f_t \in [\frac 1n, 1-\frac 1n]^n$. 
  \begin{enumerate} 
  \item\label{it:boundaryL} Let $L \coloneqq \{i \in [1..n] \mid f_{it} = \frac 1n\}$, $\ell = |L|$, and $M \coloneqq \{i \in L \mid x^1_i \neq x^2_i\}$. Then $|M| \sim \Bin(\ell,P)$ and $\|f_{t+1}\|_1 - \|f'_{t+1}\|_1 \preceq \|(f_{t+1})_{|L}\|_1 - \|(f'_{t+1})_{|L}\|_1 \preceq \tfrac 1\mu |M| \preceq \tfrac 1\mu \Bin(n,\tfrac 2n)$. 
  \item\label{it:boundaryU} Let $L \coloneqq \{i \in [1..n] \mid f_{it} = 1 - \frac 1n\}$, $\ell = |L|$, and $M \coloneqq \{i \in L \mid x^1_i \neq x^2_i\}$. Then $|M| \sim \Bin(\ell,P)$ and $\|f'_{t+1}\|_1 - \|f_{t+1}\|_1 \preceq \|(f'_{t+1})_{|L}\|_1 - \|(f_{t+1})_{|L}\|_1 \preceq \tfrac 1\mu |M| \preceq \tfrac 1\mu \Bin(n,\tfrac 2n)$.
  \end{enumerate} 
\end{lemma}

\begin{proof}
By symmetry, it suffices to prove the first part. For an $i \in L$, we have $\Pr[x^1_i \neq x^2_i] = 2 \frac 1n (1-\frac 1n) = P$. Since the bits of $x^1$ and $x^2$ were sampled independently, we have $|M| \sim \Bin(\ell,P)$.

By the well-behaved frequency assumption and the fact that $f'_{t+1} = f_t + \frac{1}{\mu} (y^1 - y^2)$ for binary vectors $y^1$ and $y^2$, we can have $f'_{i,t+1} < \frac 1n$ and thus $f_{i,t+1} > f'_{i,t+1}$ only when $f_{it} = \frac 1n$ and $x^1_i \neq x^2_i$, that is, when $i \in M$. This shows $\|f_{t+1}\|_1 - \|f'_{t+1}\|_1 \preceq \|(f_{t+1})_{|L}\|_1 - \|(f'_{t+1})_{|L}\|_1$. 

Since $f_{i,t+1} > f'_{i,t+1}$ implies $f_{i,t+1} = f'_{i,t+1} + \frac 1\mu$, we also have $\|(f_{t+1})_{|L}\|_1 - \|(f'_{t+1})_{|L}\|_1 \preceq \frac 1\mu |M| \preceq \tfrac 1\mu \Bin(n,\tfrac 2n)$.  
\end{proof}

Since sampling the optimum is particularly unlikely when frequencies are close to the lower boundary, we shall argue that the frequencies in a run of the cGA on \onemax stay away from the lower boundary for a decent time. 

A similar result was given in~\cite[Lemma~2.4]{HasenohrlS18}, however, the proof appears to be not complete. It seems to us that the main technical prerequisite of this result, Lemma~2.2~in~\cite{HasenohrlS18} with a proof of a little over one page in the condensed proceedings style, is not correct for two reasons. Since the proof of Lemma~2.2 never refers to the frequency boundaries, it is not clear if it is applicable for the cGA with these boundaries. Rather, a frequency vector having one entry $f_{it}=\frac 1n$ and another one $f_{jt}=1 - \frac 1n$ seems to be a counter-example (note that the frequency vector is called $p_t$ instead of $f_t$ in~\cite{HasenohrlS18}). However, also for the case without boundaries counter-examples seem to exist for all value of $\mu$, e.g., the frequency vector $f_t = (\frac 1 {100}, \frac 12)$. 

We did not see how to repair the otherwise elegant argument via the Azuma-Hoeffding inequality. For this reason, using a sequence of elementary reductions, we argue that the true random process of a frequency, which is not a Markov process when regarding one frequency in isolation, can be pessimistically replaced by a fair random walk on an unbounded frequency domain. For the analysis of the latter, classic Chernoff bounds can be used. This general approach was also taken in~\cite{Droste06}, however in the easier situation that there are no frequency boundaries and that the objective function is $\onemax$.

\begin{lemma}\label{lconc}
  Let $\mu$ be arbitrary except that it satisfies the well-behaved frequency assumption. Let $\eps > 0 $. Let $Z_0, Z_1, \dots$ be any random process on $F_\mu$ such that (i)~$Z_0 = \frac 12$, (ii)~for all $t = 0, 1, \dots$ there are numbers $p_t, q_t, r_t \in [0,1]$, depending on $Z_0, Z_1, \dots, Z_t$, such that $p_t + q_t + r_t = 1$ and, conditional on $Z_0, \dots, Z_t$,
  \begin{align*}
  \Pr[Z_{t+1} = Z_t] &=p_t\\
  \Pr[Z_{t+1} = Z_t + \tfrac 1 \mu] &= q_t\\
  \Pr[Z_{t+1} = Z_t - \tfrac 1 \mu] &= r_t.
  \end{align*}
  We further assume that $r_t = 0$ when $Z_t = \frac 1n$, that $q_t = 0$ if $Z_t = 1 - \frac 1n$, and that $q_t \ge r_t$ when $Z_t \neq 1 - \frac 1n$. 
  Then for all $T \in \N$, 
  \[\Pr[\exists t \in [0..T] : Z_t < \tfrac 12 - \eps] \le 2 \exp\left(-\frac{2 \mu^2 \eps^2}{T}\right).\]
\end{lemma}

\begin{proof}
  We first observe that we can assume $p_t = 0$ for all $t$. The event $Z_{t+1} = Z_t$ that the process does not move only slows down the process in the sense that it visits fewer states. Similarly, we can assume that $q_t = r_t$ except in the cases $Z_t \in \{\frac 1n, 1-\frac 1n\}$. For this now uniquely defined process, which is an unbiased random walk with reflecting boundaries, we show
\[\Pr[\exists t \in [0..T] : Z_t \notin [\tfrac 12-\eps,\tfrac 12 +\eps]] \le 2 \exp\left(-\frac{2 \mu^2 \eps^2}{T}\right).\]

 Being interested in the event that the process reaches a state outside ${[\tfrac 12-\eps,\tfrac 12 +\eps]}$ at least once, we can also drop the boundary conditions and assume that we have $Z_{t+1} \in \{Z_t - \frac 1 {\mu}, Z_t + \frac 1 \mu\}$ uniformly at random at all times $t$. We can now rewrite the $Z_t$ as follows. Let $X_1, \dots, X_T$ be independent random variables uniformly distributed on $\{-\frac 1 \mu, \frac 1 \mu\}$. Then for all $t$, $Z_t$ has the same distribution as $\frac 12 + \sum_{i=1}^t X_t$. Consequently, by the additive Chernoff bound (in the sharper version working also for maxima, see~(2.17) and Theorem~2 in~\cite{Hoeffding63} or, e.g., Theorem~10.31 together with Theorem~10.9 in~\cite{Doerr18bookchapter}), we have 
\begin{align*}
\Pr[\exists &t \in [0..T] : Z_t \notin [\tfrac 12-\eps,\tfrac 12 +\eps]] \\
& = \Pr[\exists t \in [0..T] : |Z_t - E[Z_t]| > \eps] \\
&\le 2 \exp\left(-\frac{2 \eps^2}{T (1/\mu)^2}\right) = 2 \exp\left(-\frac{2 \mu^2 \eps^2}{T}\right).
\end{align*}  
\end{proof}

The following result is a weaker form of what was shown in the proof of Lemma~5 in~\cite{Droste06}.

\begin{lemma}\label{ldroste}
  There is a constant $C > 0$ such that the following holds. Let $n \in \N$ and $D \in \N$. Let $f \in [\frac 13,1]^n$ such that $\|f\|_1 \le n - D$. Let $x^1, x^2 \sim \Sample(f)$ independently. Then 
  \[\Pr\left[\left|\|x^1\|_1 - \|x^2\|_1\right| \ge \tfrac 15 \sqrt{D}\right] \ge C.\]
\end{lemma}

Since we shall use that the optimization process of the cGA on a jump function is identical to the one on the \onemax function as long as no search point in the gap region is sampled, we find the following analysis of the optimization process on \onemax useful. It differs from Droste's analysis of the cGA on \onemax~\cite{Droste06} in that it regards the cGA with boundaries and in that it proves a high-probability statement for reaching a near-optimal frequency vector. We note that our analysis can easily be extended to also give a bound for the time to sample an optimal solution, but we do not need such a result (and in fact, such a result is implied by our main result). Also, a simplified version of our proof would apply to the cGA without boundaries.

\begin{lemma}\label{lonemax}
  Consider a run of the cGA with $\mu \ge \log_2 n$ on the \onemax benchmark function. Let $D_t \coloneqq n - \|f_t\|_1$ for all $t$. Let $K$ be a sufficiently large constant. Let $T$ be the first time that $D_t \le K$ or that there is an $i \in [1..n]$ with $f_{it} < \frac 13$. Then
  \[\Pr\left[T \ge \frac{10(2+\sqrt 2)}{C}\,\mu \sqrt n \right] = \exp(-\Omega(\mu)),\]
  where $C$ is the constant from Lemma~\ref{ldroste}.
\end{lemma}

\begin{proof}
Consider a run of the cGA on \onemax. Define $D'_t \coloneqq n - \|f'_t\|_1$ for all $t \ge 1$.
Since we are done when we reach a frequency below $\frac 13$, we can in the following assume whenever convenient that $f_t \in [\frac 13,1]^n$. To be very precise, we note that we do not condition on this event, since the conditional probability space is harder to work with, among others, because there the bit values of the offspring are not sampled independently.

For $i = 1, 2, \dots$ let $d_i = 2^{-i} n$. Without loss of generality, we may assume that $K = 2^{-\ell-1}n$ for some $\ell \in \N$. Note that $\ell \le \log_2 n$. We say that the optimization process enters Phase $i$ (and thus leaves its current phase) when for the first time $D_t \le d_i$. Note that we stay in Phase $i$ even when after entering this phase $D_t$ increases beyond $d_i$. Note further that, by definition, the process starts in Phase~1. 

We analyze the time spend in Phase $i \le \ell$ and show that this time, with probability at least $1 - \exp(-\Omega(\mu))$, is at most $T_i = \lceil 20 \frac 1C \mu \sqrt{d_{i+1}} \rceil$. Let $t'$ be the iteration in which the process enters Phase $i$. To ease the argument, we now consider exactly $T_i$ iterations. In case the phase ends earlier, we shall from that point on regard an artificial process, with a slight abuse of notation also denoted by $D_t$ and $D'_t$, that satisfies the conditions 
\begin{align*}
&\Pr[D'_{t+1} = D_t - \tfrac 15 \sqrt{d_{i+1}} /\mu \mid D_t] = C,\\
&\Pr[D'_{t+1} = D_t \mid D_t] = 1-C,\\
&\Pr[D_{t+1} = D'_{t+1} \mid D'_{t+1}] = 1.
\end{align*}
Such an artificial extension of a process was, to the best of our knowledge, in the theory of evolutionary algorithms first used in~\cite{DoerrHK11}.

When all frequencies are at least $\frac 13$, by Lemma~\ref{ldroste} we have $\Pr[|\|x^1\|_1 - \|x^2\|_1| \ge \tfrac 15 \sqrt{D_t}] \ge C$ for an absolute constant $C$. Since we have $\|y^1\|_1 \ge \|y^2\|_1$ when optimizing \onemax, we have that $D'_{t+1}$ with probability at least $C$ satisfies $D'_{t+1} \le D_t - \tfrac 15 \sqrt{D_t} / \mu \le D_t - \tfrac 15 \sqrt{d_{i+1}} / \mu$. We call this a \emph{success}. Note that the probability for a success is at least $C$ regardless of what happened before in this phase. Consequently, in $T_i$ iterations, we not only have an expected number of at least $20 \mu \sqrt{d_{i+1}}$ successes, but by Lemma~11 of~\cite{DoerrJ10} and the multiplicative Chernoff bounds we also have at least $10 \mu \sqrt{d_{i+1}}$ successes with probability at least $1 - \exp(-\tfrac 52 \mu \sqrt{d_{i+1}})$. Note that with probability one we have $D'_{t+1} \le D_t$, again because $\|y^1\|_1 \ge \|y^2\|_1$.

By Lemma~\ref{lboundary}~\ref{it:boundaryU}, we have $D_{t+1} \preceq D'_{t+1} + \Bin(n, \frac 2n)$, again regardless of what happened in earlier iterations. Consequently, the total number of times we increase $D_t$ due to reaching the upper boundary can be estimated by a sum of $T_i n$ independent binary random variables with success probability $\frac 2n$. Hence the expectation of this number is at most $2 T_i \le 40 \frac 1C \mu \sqrt{d_{i+1}} + 2$ and with probability at least $1 - \exp(-\frac{40}{3} \frac 1C \mu \sqrt{d_{i+1}})$ this number is at most $4 T_i = 80 \frac 1C \mu \sqrt{d_{i+1}}+4$.

Taking these two observations together, we see that with probability \[1 - \exp\left(-\frac 52 \mu \sqrt{d_{i+1}}\right)  - \exp\left(-\frac{40}{3} \frac 1C \mu \sqrt{d_{i+1}}\right)  = 1 - \exp\left(-\Omega(\mu)\right),\] we have 
\begin{align*}
D_{t'+T_i} &\le D_{t'} - 10 \mu \sqrt{d_{i+1}} \cdot \tfrac 15 \sqrt{d_{i+1}} / \mu + (80 \tfrac 1C \mu \sqrt{d_{i+1}}+4) / \mu \\
&= D_{t'} - 2 d_{i+1} + \tfrac{80}{C} \sqrt{d_{i+1}} + 4/\mu.
\end{align*} 
Since $K = 2^{-\ell-1} n \le d_{i+1}$ was chosen sufficiently sufficiently large, we have $D_{t'+T_i} \le D_{t'} - d_{i+1}$, that is, $D_{t'+T_i}$ belongs to a later phase already. Consequently, we have that with probability at least $1 - \exp(-\Omega(\mu))$, at most $T_i$ rounds are spend in Phase~$i$. 

We show our claim by computing 
\begin{align*}
\sum_{i=1}^{\ell} T_i &\le \ell + \sum_{i=1}^{\ell} 20 \frac 1C \mu \sqrt{2^{-(i+1)}n} \le \frac{10}{C}\mu \sqrt n \sum_{i=0}^\infty (2^{-1/2})^i \\
&= \frac{10}{C}\mu \sqrt n \frac{1}{1-2^{-1/2}} = \frac{10(2+\sqrt 2)}{C}\mu \sqrt n.
\end{align*}
\end{proof}

We now analyze the drift in $D_t$ when we are that close to the gap that we cannot assume anymore that we never sample a search point in the gap. To be precise, let us define the gap by 
\[G \coloneqq G_{nk} \coloneqq \{x \in \{0,1\}^n \mid n-k < \|x\|_1 < n\}.\] 
Let $G^+ \coloneqq G \cup \{(1,\dots,1)\}$.

A difficulty here, which was not treated fully rigorously in~\cite[Lemma~3.1]{HasenohrlS18}, is that the event $G_t$ that $x^1$ or $x^2$ lie in the gap and the random variable $|\|x^1\|_1 - \|x^2\|_1|$ are not independent. Consequently, the estimate $E[D_t - D_{t+1} \mid D_t] =  \frac 1\mu |\|x^1\|_1 - \|x^2\|_1| (1 - 2\Pr[G_t])$ is not correct. In fact, the correlation is indeed not in our favor. When $|\|x^1\|_1 - \|x^2\|_1|$ is large, the probability that a search point in the gap was sampled (and thus the frequency update is done in the unwanted direction) is higher. 

\begin{lemma}\label{ldrift}
  Let $\mu$ be arbitrary satisfying the well-behaved frequency assumption. Let $k \in [1..\frac 12 n - 1]$. Consider an iteration $t$ of the cGA optimizing $\jump_{nk}$ started with a frequency vector $f_t$ such that $D_t = n - \|f_{t}\|_1 \ge 2k$ and such that $f_{it} \ge \frac 13$ for all $i \in [1..n]$. Then 
  \[E[\mu D_{t} - \mu D_{t+1}]  \ge \tfrac{1}{5} C \sqrt{D_t} - 6 D_t \exp(-\tfrac 18 D_t) -2,\]
  where $C$ is the constant from Lemma~\ref{ldroste}.
\end{lemma}

\begin{proof}
From the definition of the cGA, we note that when $x^1$ and $x^2$ are both not in $G^+$, then $D'_{t+1} \coloneqq n - \|f'_{t+1}\|_1$ satisfies $D'_{t+1} = D_t - \frac 1\mu |\|x^1\|_1 - \|x^2\|_1|$. In all other cases, we have $D'_{t+1} \le D_t + \frac 1\mu |\|x^1\|_1 - \|x^2\|_1|$. Consequently, 
\begin{align*}
&E[\mu D_{t} - \mu D'_{t+1}]\\ 
&\ge \Pr[x^1, x^2 \notin G^+] \, E[|\|x^1\|_1 - \|x^2\|_1| \mid x^1, x^2 \notin G^+] \\
& \quad - \Pr[\{x^1, x^2\} \cap G^+ \neq \emptyset] \, E[|\|x^1\|_1 - \|x^2\|_1| \mid \{x^1, x^2\} \cap G^+ \neq \emptyset]\\
&= E[|\|x^1\|_1 - \|x^2\|_1|] \\
& \quad - 2 \Pr[\{x^1, x^2\} \cap G^+ \neq \emptyset] \, E[|\|x^1\|_1 - \|x^2\|_1| \mid \{x^1, x^2\} \cap G^+ \neq \emptyset].
\end{align*}
When the frequencies are all at least $\frac 13$, we conclude from Lemma~\ref{ldroste} that $E[|\|x^1\|_1 - \|x^2\|_1|] \ge \frac{1}{5} C \sqrt{D_t}$. 

For the contribution when search points are in~$G^+$, we first note that the second bound of Lemma~\ref{lsample} (with $\delta = \frac 12$ and $D^- = D_t$) and $D_t \ge 2k$ yield 
\begin{equation*}
\Pr[x^1 \in G^+] \le \Pr[d(x^1) \le \tfrac 12 D_t] \le \exp(-\tfrac 18 D_t).
\end{equation*}
Then, exploiting the symmetry between $x^1$ and $x^2$, counting the case $x^1, x^2 \in G^+$ twice, and using again $\frac 12 D_t \ge k$, we compute
\begin{align*}
\Pr[&\{x^1, x^2\} \cap G^+ \neq \emptyset] \, E[|\|x^1\|_1 - \|x^2\|_1| \mid \{x^1, x^2\} \cap G^+ \neq \emptyset] \\
& \le 2 \Pr[x^1 \in G^+] \, E[|\|x^1\|_1 - \|x^2\|_1| \mid x^1 \in G^+] \\
& \le 2 \Pr[x^1 \in G^+] \, \left(E[|\|x^1\|_1 - n| \mid x^1 \in G^+] + E[|n - \|x^2\|_1|]\right) \\
& \le 2 \Pr[x^1 \in G^+] \, \left(k + D_t\right) \\
& \le 2 \exp(-\tfrac 18 D_t) (\tfrac 12 D_t + D_t) = 3 \exp(-\tfrac 18 D_t) D_t.
\end{align*}

In summary, we have
\[E[\mu D_{t} - \mu D'_{t+1}] \ge \tfrac{1}{5} C \sqrt{D_t} - 6 D_t \exp(-\tfrac 18 D_t).\]

By Lemma~\ref{lboundary}, we further have $E[\mu D_{t+1} - \mu D'_{t+1}] \le 2$. Consequently, recalling that the linearity of expectation holds also for dependent random variables, we have 
\begin{align*}
E[\mu D_{t} - \mu D_{t+1}] &= E[\mu D_{t} - \mu D'_{t+1}] - E[\mu D_{t+1} - \mu D'_{t+1}] \\
&\ge \tfrac{1}{5} C \sqrt{D_t} - 6 D_t \exp(-\tfrac 18 D_t) -2.
\end{align*}
\end{proof}

The following elementary estimate gives a lower bound for the probability to sample the optimum.
\begin{lemma}\label{lopt}
  Let $0 < c < 1$ and $f \in [c,1]^n$. Let $x \sim \Sample(f)$. Then $\Pr[x = (1, \dots, 1)] \ge c^{(n - \|f\|_1)/(1-c)}$.
\end{lemma}

\begin{proof}
  For $i \in [1..n]$, let $\alpha_i \coloneqq \frac{1-f_i}{1-c\phantom{_i}}$. Then $f_i = \alpha_i c + (1-\alpha_i) 1$ is the unique representation of $f_i$ as convex combination of $c$ and $1$. Since the logarithm is concave, we have 
  \[\log f_i = \log(\alpha_i c + (1-\alpha_i) 1) \ge \alpha_i \log c + (1-\alpha_i) \log 1 = \log(c^{\alpha_i} 1^{1-\alpha_i}).\] 
  Since the logarithm is monotonically increasing, this inequality implies ${f_i \ge c^{\alpha_i} 1^{1-\alpha_i}} = c^{\alpha_i}$. Consequently, 
  \begin{align*}
  \Pr[x = (1, \dots, 1)] & = \prod_{i=1}^n f_i  \ge \prod_{i=1}^n c^{\alpha_i} = c^{\sum_{i=1}^n \alpha_i} = c^{(n - \|f\|_1)/(1-c)}.
  \end{align*}
\end{proof}

We shall use the following estimate on the expected change of a frequency that is not affected by the boundaries. This result was proven in~\cite[Lemma~3]{SudholtW16}.

\begin{lemma}\label{lonemax2}
  Let $\mu$ be arbitrary. Consider a run of the cGA optimizing \onemax. Consider an iteration starting with a frequency vector $f_t$. Let $i \in [1..n]$ be such that $\frac 1n + \frac 1\mu \le f_{it} \le (1 - \frac 1n) - \frac 1\mu$. Then
  \[E[f_{i,t+1} - f_{it}] \ge \frac{2}{11} \frac{f_{it} (1-f_{it})}{\mu} \left(\sum_{j \neq i} f_{jt} (1-f_{jt})\right)^{-1/2}.\]
\end{lemma}

\subsection{Main Result and Proof}

We are now ready to state and prove our main result.

\begin{theorem}
  Let $k \le \frac 1 {20} \ln(n)-1$. Let $\mu \ge K \sqrt n \ln(n)$ for a sufficiently large constant $K$, but polynomially bounded in $n$. Then the cGA with frequency boundaries (Algorithm~\ref{alg:cga}) with hypothetical population size $\mu$ with probability $1 - o(1)$ finds the optimum of the $\jump_{nk}$ function in time $O(n \log n)$.
\end{theorem}

\begin{proof}
  To allow the reader to easily check that all implicit constants can be chosen in a way that they give the claimed results, we make these constants explicit in the following proof, but note that for most of them it just suffices to choose them sufficiently large. 

  Let $k \le C_k \ln(n)-1$, $C_k = \frac 1 {20}$.
  Let $\mu \ge c_{\mu} \sqrt n \ln(n)$ and $\mu \le n^{C_\mu}$ for a constant $c_\mu$ to be defined in a moment and, say, $C_\mu \ge 1$.
  Consider a run of the cGA on the objective function $\jump_{nk}$. 
  
  Let $\tilde T'$ be the first time that $D_t := n - \|f_t\|_1$ satisfies $D_t \le D' \coloneqq C_{D'} \ln n$, where $C_{D'} \ge 8 C_\mu + 12$ is a constant.  
  Let $\tilde T''$ be the first time that $D_t \le D'' \coloneqq \max\{2k+1, C_{D''}\}$, where $C_{D''}$ is a sufficiently large constant (independent also of all other constants). 
  
  Let $T' = \min\{\tilde T', \lfloor C_{T'} \mu \sqrt n \rfloor\}$ with $C_{T'} = \frac{10(2 + \sqrt 2)}{C}$, where $C$ is the constant from Lemma~\ref{ldroste}. Let $T'' = \min\{\tilde T'', \lfloor C_{T''} \mu \sqrt n \rfloor\}$ with $C_{T''} = C_{T'}+1$. Let now $c_\mu \ge 36 C_{T''}$.
  
  We first argue that with high probability we have no frequencies below~$\tfrac 13$ up to time $T''$. For this, consider some time $t$ such that $f_t \in [\frac 13,1]^n$ and $D_t \ge D''$. Consider a fixed bit $i\in [1..n]$ such that $f_{it} \neq 1-\frac 1n$. If we were optimizing the \onemax function, then by Lemma~\ref{lonemax2},
  \begin{align*}
  \Pr[&f_{i,t+1} = f_{it} + \tfrac 1\mu] - \Pr[f_{i,t+1} = f_{it} - \tfrac 1\mu]\\
  & = \mu E[f_{i,t+1} - f_{it}]\\
  & \ge \frac{2}{11} f_{it}(1-f_{it}) \left(\sum_{j\neq i} f_{jt}(1-f_{jt})\right)^{-1/2}\\
  & \ge \frac{2}{11} f_{it}(1-f_{it}) \left(D_t\right)^{-1/2}.
  \end{align*}
  
  Regardless of whether we optimize \onemax or $\jump_{nk}$, the events $f_{i,t+1} = f_{it} + \tfrac 1\mu$ and $f_{i,t+1} = f_{it} - \tfrac 1\mu$ can only occur when the two search points sampled in this iteration satisfy $x^1_i \neq x^2_i$. Further, the definition of $f_{i,t+1}$ differs from the \onemax case at most when at least one of $x^1$ and $x^2$ lie in the gap $G_{nk}$. Hence the following coarse correction of the above estimate is valid for the optimization of $\jump_{nk}$.
  \begin{align*}
  &\Pr[f_{i,t+1} = f_{it} + \tfrac 1\mu] - \Pr[f_{i,t+1} = f_{it} - \tfrac 1\mu]\\
  & \ge \tfrac{2}{11} f_{it}(1-f_{it}) \left(D_t\right)^{-1/2} - \Pr[(x^1_i \neq x^2_i) \wedge (\{x^1,x^2\} \cap G_{nk} \neq \emptyset)].
  \end{align*}
  We now estimate this correction term, first by noting that $\Pr[(x^1_i \neq x^2_i) \wedge (\{x^1,x^2\} \cap G_{nk} \neq \emptyset)] = \Pr[x^1_i \neq x^2_i] \cdot \Pr[\{x^1,x^2\} \cap G_{nk} \neq \emptyset \mid x^1_i \neq x^2_i]$, then by using the union bound estimate $\Pr[\{x^1,x^2\} \cap G_{nk} \neq \emptyset \mid x^1_i \neq x^2_i] \le 2 \Pr[x^1 \in G_{nk} \mid x^1_i \neq x^2_i]$.
  Conditional on $x^1_i \neq x^2_i$, the bit string $x^1$ is sampled from $\Sample(f_t)$, however, conditional on the $i$-th bit being zero or one. In either case, to have $x^1 \in G_{nk}$, we need that $\tilde D = \sum_{j \neq i} (1-x^1_j)$ is at most $k \le \frac 12 (D_t-1)$, where we recall that $D_t \ge D'' \ge 2k+1$. Since $E[\tilde D] = D_t - (1-f_{it}) \ge D_t-1$, by Lemma~\ref{lsample} with $\delta = \frac 12$ this event happens with probability at most $\exp(-\tfrac 18 (D_t-1))$. Together with $\Pr[x^1_i \neq x^2_i] = 2 f_{it}(1-f_{it})$, we obtain
  \begin{align*}
  \Pr[&f_{i,t+1} = f_{it} + \tfrac 1\mu] - \Pr[f_{i,t+1} = f_{it} - \tfrac 1\mu]\\
  & \ge \tfrac{2}{11} f_{it}(1-f_{it}) \left(D_t\right)^{-1/2} - 2 f_{it}(1-f_{it}) \exp(-\tfrac 18 (D_t-1)),
  \end{align*}
  which is non-negative since $D_t \ge D'' \ge C_{D''}$, which was chosen sufficiently large. 
  
  Consequently, the process $(f_{it})_t$ satisfies the assumptions of Lemma~\ref{lconc} up to time $T''$. If $T'' < C_{T''} \mu \sqrt n$, we artificially extend the process (for the following argument only) by setting $f_{it} = f_{i T''}$ for all $t \in [T''+1..C_{T''} \mu \sqrt n]$. By Lemma~\ref{lconc} we thus obtain that up to time $T = \lfloor C_{T''} \mu \sqrt n \rfloor$, the \mbox{$i$-th} frequency is always at least $\frac 13$ with probability $1 - 2\exp(-\frac{\mu^2}{18T}) \ge 1 - 2\exp(-\frac{\mu}{18 C_{T''} \sqrt n}) \ge 1 - 2\exp(-\frac{c_{\mu}}{18 C_{T''}} \ln n)$. With a union bound over the $n$ frequencies, we have $f_t \in [\frac 13,1]^n$ in this time interval with probability at least $1 - 2n\exp(-\frac{c_{\mu}}{18 C_{T''}} \ln n) = 1 - O(1/n)$ by choice of $c_\mu$ and $C_{T''}$.

Since $D' = C_{D'} \ln n$ with $C_{D'} \ge 12$ and $k \le C_k \ln n \le \frac{C_{D'}}2 \ln n$, by Lemma~\ref{lsample} and a union bound the probability that within the first $T' \le C_{T'}\mu\sqrt n$ iterations a search point in the gap region is sampled, is at most $2 C_{T'}\mu\sqrt n \exp(-\frac {C_{D'}}8 \ln n) \le 2 C_{T'} n^{C_\mu + 0.5 - C_{D'}/8} = O(1/n)$. Consequently, by Lemma~\ref{lonemax}, after at most $C_{T'}\mu\sqrt n$ iterations with probability $1 - O(1/n)$ we have $D_t \le D'$. 

We now estimate the additional time it takes to reach $D_t \le D''$. Let $t_0$ be the first time such that $D_t \le D'$. By Lemma~\ref{ldrift} and using our assumption that $C_{D''}$ is a large absolute constant, we have $E[D_t - D_{t+1} \mid D_t] \ge \frac{1}{\mu}$ when $D_t \ge D'' \ge C_{D''}$. We define a random process $\tilde D_t$ as follows. Let $t \ge t_0$. If $D_s < D''$ for some $s \in [t_0..t]$, then $\tilde D_{t-t_0} = 0$. Otherwise, $\tilde D_{t - t_0} = D_t$. By the above observation, we have $E[\tilde D_t - \tilde D_{t+1} \mid \tilde D_t > 0] \ge \frac{1}{\mu}$. By the additive drift theorem~\cite{HeY01}, also to be found in the recent survey~\cite{Lengler17}, $T\coloneqq \min\{t \mid \tilde D_t = 0\}$ satisfies $E[T] \le \frac{D'}{1/\mu} = O(\mu \log n)$. By Markov's inequality, we have $T = O(\mu n^{0.4} \log n)$ with probability $1 - n^{-0.4}$. 

Let now $t_0$ be such that $D_{t_0} \le D''$ and all frequencies are at least $\frac 13$. We first argue that if $D_t \le \ln(n)^2$, then $\Pr[D_{t+1} \ge D_t + \frac 4\mu \ln(n)^2] \le n^{\omega(1)}$. By Lemma~\ref{lsample}, we have $\Pr[d(x^j) \ge 2\ln(n)^2] \le \exp(-\ln(n)^2/3) = n^{\omega(1)}$ for $j = 1,2$. Consequently, with probability $1 - n^{\omega(1)}$, we have both $\|x^1\|_1 \ge n - 2\ln(n)^2$ and $\|x^2\|_1 \ge n - 2\ln(n)^2$. In this case, the Hamming distance between $x^1$ and $x^2$ satisfies $H(x^1,x^2) \le 4 \ln(n)^2$, which implies that $|D_t - D_{t+1}| \le \|f_t - f_{t+1}\|_1 \le \frac 4 \mu \ln(n)^2$ and thus $D_{t+1} \le D_t + \frac 4\mu \ln(n)^2$. By a union bound, with probability $1 - n^{\omega(1)}$, this happens in all iterations $t_0, \dots, t_0 + \lfloor\frac{D''}{(4/\mu) \ln(n)^2}\rfloor -1$ and consequently, throughout these $L = \lfloor \frac{D''}{(4/\mu) \ln(n)^2} \rfloor$ iterations we have $D_t \le 2D''$. Note that $L = O(\mu / \log(n))$, hence throughout this period we also have $f_{it} \ge \tfrac 13 - \tfrac 1 \mu L \ge 0.32$ (assuming $n$ to be sufficiently large). By Lemma~\ref{lopt}, the probability that a fixed search point sampled in this period is the optimum, is at least $0.32^{2D'' / 0.68} \ge 0.32^{4 C_k \ln(n) / 0.68} = \exp(4 C_k \ln(n) \ln(0.32) / 0.68) \ge n^{-6.71 C_k} \ge n^{-0.34}$ by choice of $C_k$. Hence the probability that the optimum is not sampled in this period is at most $(1 - n^{-0.34})^{2L} \le (1 - n^{-0.34})^{\mu / \ln(n)^2} \le \exp(- n^{-0.34} \cdot \mu / \ln(n)^2) \le \exp(-\Omega(n^{0.16}/ \log(n)))$.
\end{proof}

Let us remark that we did not try to optimize the implicit constants, nor did we try to find the largest constant $C_k$ such that the $O(n \log n)$ runtime guarantee holds for all $k \le C_k \ln(n) - 1$. We further note that all  but one argument in the above proof, by choosing the constants right, would give a success probability of $1 - n^{-c}$, where $c$ can be any constant. This is not true for the Markov bound argument in the analysis of the time to reach a $D_t$ value of at most $D''$. Without further details, we note that also for this phase an arbitrary inverse-polynomial failure probability could be obtained with stronger methods. 

Finally, we note that by taking $k=1$, our result also applies to the \onemax function.

\section{Conclusion}

This is, to the best of our knowledge, only the second mathematical analysis of an EDA on a multi-modal optimization problem. Our main result shows that the cGA can optimize jump functions with logarithmic jump sizes in asymptotically the same efficiency as the simple \onemax function. It thus does not suffer from the fitness valleys present in these objective function.

The obvious question arising from this work is to what extent such or similar results hold for other EDAs. Natural candidates could be the UMDA, for which several rigorous runtime results exist, see~\cite{KrejcaW18}, and the significance-based cGA~\cite{DoerrK18}, which might profit from using only the three frequencies $\frac 1n$, $\frac 12$, and $1 - \frac 1n$. Equally interesting would be results for other multi-modal optimization problems. On the more speculative side, given that the black-box complexity of jump functions is low even for large jump sizes~\cite{BuzdalovDK16}, one could also try to challenge the upper bound $\exp(O(k))$ given in~\cite{HasenohrlS18} for larger values of~$k$, either by proving that the cGA also performs better here or by exploring if suitable modifications of the cGA can lead to a better performance.

}


\newcommand{\etalchar}[1]{$^{#1}$}

\end{document}